\documentclass{article} 
\usepackage{iclr2027_conference,times}

\usepackage{amsmath,amsfonts,bm}

\def\eqref#1{equation~\ref{#1}}

\def\1{\bm{1}}

\DeclareMathAlphabet{\mathsfit}{\encodingdefault}{\sfdefault}{m}{sl}
\SetMathAlphabet{\mathsfit}{bold}{\encodingdefault}{\sfdefault}{bx}{n}

\usepackage{booktabs}
\usepackage{multirow}
\usepackage{graphicx}

\usepackage[
  colorlinks=true,
  linkcolor=black,
  citecolor=black,
  urlcolor=black
]{hyperref}
\usepackage{url}
\usepackage{enumitem}

\usepackage{amsmath}
\usepackage{amssymb}
\usepackage{mathtools}
\usepackage{amsthm}
\usepackage{subcaption}         

\usepackage{algorithm}
\usepackage{algpseudocode}
\usepackage{float}

\theoremstyle{plain}
\newtheorem{theorem}{Theorem}[section]

\newtheorem{lemma}[theorem]{Lemma}

\theoremstyle{definition}

\theoremstyle{remark}

\usepackage{color}
\usepackage{url}
\usepackage{tikz}
\usepackage{dsfont}

\title{Learning with a Single Rollout via Pass@k Policy Optimization}

\author{%
    \begin{minipage}{\textwidth}
    \vspace{0.4em}
    \raggedright
    \small
    \setlength{\tabcolsep}{0pt}
    
    \begin{tabular}{@{}l@{}}
        Fengdi Che\textsuperscript{1}\quad
        Yang Liu\textsuperscript{2}\quad
        Lei Yu\textsuperscript{3}\quad
        Meng Cao\textsuperscript{4}\quad
        Tong Che\textsuperscript{5} \vspace{0.2em}\\
        A.\ Rupam Mahmood\textsuperscript{1,6,7}\quad
        Dale Schuurmans\textsuperscript{1,6,7}\vspace{0.4em}\\
        \footnotesize
            \begin{tabular}{@{}l@{\hspace{1.6em}}l@{\hspace{1.6em}}l@{\hspace{1.6em}}l@{}}
            \textsuperscript{1}University of Alberta &
            \textsuperscript{2}BIGAI &
            \textsuperscript{3}University of Toronto \vspace{0.2em} \\
            \textsuperscript{4}McGill University, Mila &
            \textsuperscript{5}Nvidia Research &
            \textsuperscript{6}Amii &
            \textsuperscript{7}CIFAR AI Chair
        \end{tabular}
        \scriptsize
    \end{tabular}
    \end{minipage}
}

\iclrfinalcopy 
\begin{document}

\maketitle

\begin{abstract}
Group Relative Policy Optimization (GRPO) typically requires multiple rollouts per prompt to estimate relative credit, incurring substantial sampling cost and exacerbating training asynchrony. This limitation is particularly severe for scientific reasoning agents, where each trajectory may involve long-horizon reasoning, multiple turns, and parallel sub-runs. We study \emph{single-rollout} reinforcement learning, where only one rollout is collected per prompt, and find that naive training can stagnate or become severely unstable. We introduce Single-Rollout Pass@$k$ Policy Optimization (SR-PPO), which combines Monte Carlo token-level advantage estimation with Pass@$k$ gradient reweighting. An online critic without large-scale pretraining is used both as a prefix-dependent baseline for credit assignment and to estimate the Pass@$k$ gradient weight. 
Using only one rollout per prompt, SR-PPO achieves stable learning and performance comparable to GRPO across reasoning and logic benchmarks, even under off-policy drift that degrades the performance of standard Pass@1 PPO.
A key factor in this stability is Pass@$k$ reweighting, which emphasizes harder prompts while downweighting high-success prompts, for which imperfect credit estimates tend to produce noisier policy updates.
Through analyses of critic quality, gradient noise, and update behavior, we investigate how Pass@$k$ reweighting stabilizes single-rollout learning and why residual critic errors may limit SR-PPO’s ability to outperform GRPO.
\end{abstract}

\section{Introduction}

Language models (LMs) have made rapid progress in complex reasoning, particularly when post-trained with reinforcement learning (RL) from verifiable rewards \citep{cobbe2021training, hendrycks2021math, lewkowycz2022minerva, shao2024deepseekmath, guo2025deepseekr1}. 
However, policy gradient estimates based on binary outcome rewards can exhibit high variance, especially under high-entropy sampling used to encourage exploration.
To stabilize optimization, Group Relative Policy Optimization (GRPO) samples multiple responses to each prompt and constructs advantages by normalizing rewards within the group \citep{shao2024deepseekmath}. This simple variance-reduction mechanism has made GRPO a widely adopted recipe for RL post-training.

Yet its reliance on multiple rollouts introduces substantial sampling cost. This limitation becomes particularly important for modern agentic systems, where a single trajectory may involve long-horizon reasoning, tool calls \citep{wu2026mcpmark}, memory retrieval \citep{jia2025the}, sub-agents \citep{wang2025mixtureofagents}, or expensive environment interactions \citep{yang2026onemillion}. Moreover, independently sampled trajectories can quickly diverge into qualitatively different prefixes, making cross-rollout comparison increasingly indirect \citep{hou2026single}.

\begin{figure}
\centering
\includegraphics[width=\linewidth]{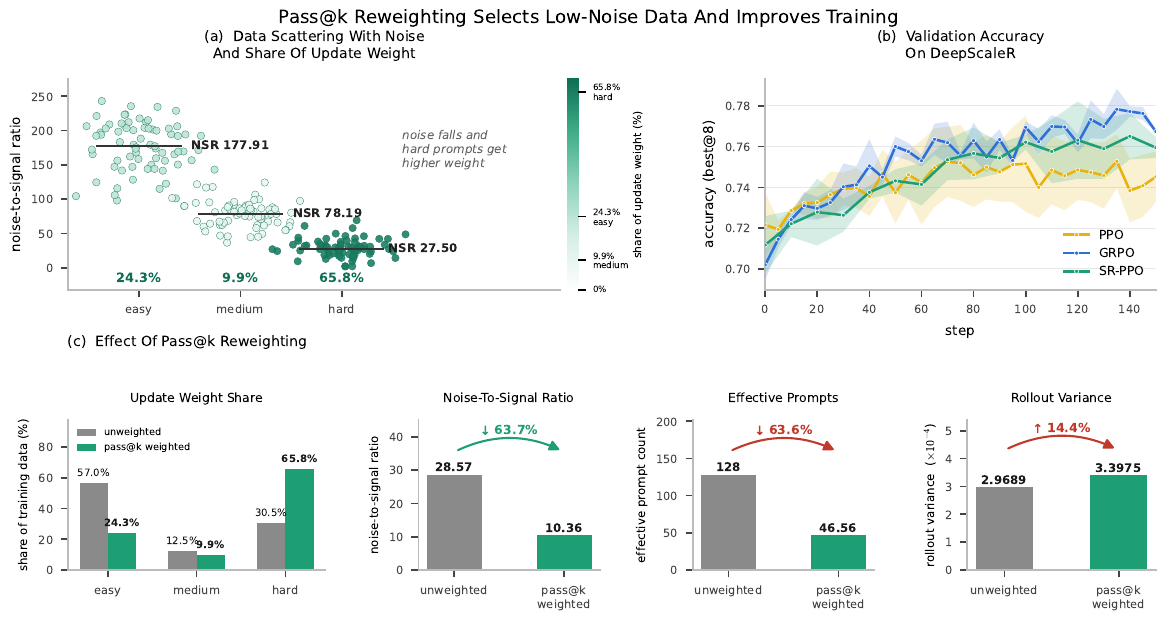}
\caption{\textbf{Method overview and performance for SR-PPO.} \textbf{(a)}~Hard prompts are defined to have at most two correct rollouts among eight independent samples from the given checkpoint. Noise-to-signal ratio measures the noise level of the gradient. The ratio reduces as the prompts get harder. Our SR-PPO gives least noisy data the highest share in the gradient update and thus stabilizes learning. \textbf{(b)}~Validation Best@8 success rates of our SR-PPO, GRPO and PPO versus training steps under a fair computation comparison. \textbf{(c)}~ After adding Pass@$k$ reweighting, hard prompts get more shares in the gradient and also noise-to-signal ratio reduces. However, the variance of gradient even increases with less effective prompts per batch, implying that noise-to-signal ratio reduction can be the true reason for performance improvement.}
\label{fig:intro}
\vspace{-2.5em}
\end{figure}

These limitations motivate the \emph{single-rollout} setting, in which only one solution is generated for each prompt. Single-rollout RL removes the requirement of repeatedly sampling the same problem, but also removes the group-relative baseline that makes GRPO stable. Consequently, estimating token-level credit in a single rollout that yields low-variance advantages becomes a central challenge.
A natural solution is to introduce a critic that predicts the return starting from each prefix, which is the probability of eventual success from intermediate prefixes in the outcome-reward setting. Then, temporal changes in the critic value can then be used to construct token-level advantages.

Standard PPO uses such an auxiliary critic for credit assignment \citep{schulman2015gae,schulman2017ppo}. 
In our experiments, however, the critic produces unreliable token-level credit estimates, which contribute to the poor performance and high variability across seeds of PPO observed in Figure~\ref{fig:intro}(b). 
Extensive critic pretraining would require additional rollouts and does not ensure that value estimates remain accurate as the policy's rollout distribution evolves. 
We therefore train the critic online using the same single-rollout data collected for policy optimization.

This choice introduces a central tension: online learning allows the critic to track an evolving policy, but its prefix-value estimates may be unreliable, particularly early in training.
We address this challenge with Single-Rollout Pass@$k$ Policy Optimization (SR-PPO), 
which combines Monte Carlo advantage estimation with Pass@$k$ reweighting to improve robustness to imperfect credit estimates. 
SR-PPO uses the critic only as a variance-reducing baseline, while the final outcome reward is still used as the unbiased return estimator. Meanwhile, we optimize the Pass@$k$ objective, where Pass@\(k\) denotes the probability that at least one of \(k\) independent continuations succeeds. With a learnt critic, we derive a novel single-rollout estimator of the Pass@\(k\) policy gradient with a prompt reweighting term. As shown in Figure~\ref{fig:intro}(c), this estimator substantially reduces gradient noise-to-signal ratio, leading to steady improvement comparable to multi-rollout GRPO in Figure~\ref{fig:intro}(b).

A key property of the Pass@\(k\) gradient is its prompt-dependent weighting: it suppresses updates from high-success prompts while placing greater emphasis on harder prompts. Empirically, harder prompts yield less noisy gradient estimates as shown in Figure \ref{fig:intro}(a), so this reweighting increases the contribution of more reliable training data and thus stabilizes the training. We hypothesize that easy prompts produce noisier gradients because failed rollouts on nearly solved problems are often mostly correct except for a few decisive mistakes, making accurate error localization especially important. If the critic misjudges these failure points, it can assign positive credit to unrecoverable mistakes. Meanwhile, downweighting easy prompts is unlikely to impair learning, as their already high success rates leave limited room for further improvement.
Pass@\(k\) therefore acts not only as an alternative policy objective, but also as a mechanism for controlling critic-induced gradient noise in single-rollout learning. More analysis is provided in Section \ref{sec:noise_analysis}.

We investigate this mechanism through analyses of critic quality and gradient behavior. 
The critic broadly tracks prompt difficulty, assigning higher values to easier prompts in agreement with
group-mean rewards, shown in Table~\ref{tab:critic-value}, yet performs near chance at localizing mistakes within individual rollouts, shown in Table~\ref{tab:critic-quality}. 
This helps explain both the source of gradient noise and why improved training stability alone may not enable SR-PPO to outperform GRPO.
Our gradient analysis further reveals that easy-prompt updates can partially cancel contributions from hard prompts in Figure~\ref{fig:heatmap-cosine}, 
even though these updates remain largely concentrated in the predefined safe region of the pretrained weights \citep{zhu2025path}. 
These findings suggest that update location alone does not explain training collapse and highlight gradient interference as a potential contributing factor.

Under off-policy training, SR-PPO improves over Pass@1 PPO on
competition-level mathematical reasoning. A smaller-batch variant
also sustains stable learning through step $600$. Across AIME
24-26 \citep{aime24,aime25,aime26}, HMMT 26 \citep{dekoninck2026matharena}, Knights~\&~Knaves \citep{xie2025memorization}, GPQA-D \citep{rein2024gpqa}, and Countdown4 \citep{park2026how},
SR-PPO remains broadly competitive with GRPO, with performance
differences small relative to seed variations.
Although imperfect critic estimates may still limit attainable
performance, these results support SR-PPO as a practical training
recipe for stable improvement using only one rollout per prompt.

\section{Background}
\label{sec:background}

Given a prompt $x$, a policy $\pi_\theta$ generates a token sequence $y_{1:T}$ of length $T$. We define the prefix state as $s_t=(x,y_{1:t})$, with $s_0\sim d$ containing only the prompt, $d$ being the data distribution and $y_t\sim\pi_\theta(\cdot\mid s_{t-1})$. A verifier returns a binary outcome reward $Y: s_T \mapsto \{0,1\}$, inducing rewards $R_t=0$ for $t<T$ and $R_T=Y$. 

\paragraph{Value Function Estimation}
Value function is defined to be the expected return starting from the current prefix. In this undiscounted setting, the value equals the probability of eventual success:
\begin{equation}
q_\pi(s_t)
:=\mathbb{E}_\pi\!\left[
\sum_{u=t+1}^{T}R_u\,\middle|\,s_t
\right]
=\mathbb{E}_\pi\!\left[
\mathds{1}[Y=1]\,\middle|\,s_t
\right]
=\mathbb{P}_\pi(Y=1\mid s_t).
\label{eq:prefix-value}
\end{equation}

For an integer $k\geq1$, Pass@$k$ value is the probability that at least one of $k$ independent continuations succeeds \citep{walder2026pass}, defined as
\begin{equation}
q_{\pi,k}(s_t)
:=\mathbb{P}_\pi\!\left(
\max_{1\leq i\leq k}Y_i=1\,\middle|\,s_t
\right)
=1-\bigl(1-q_\pi(s_t)\bigr)^k,
\label{eq:passk-value}
\end{equation}
where the outcomes $Y_1,\ldots,Y_k$ are conditionally independent given $s_t$ and sampled under $\pi$. Here, $k$ specifies the optimization target, not the training rollout count.

Since prefix transitions are deterministic, the token-level advantage is defined to be
\begin{equation}
A_\pi(s_{t-1},y_t)
:=q_\pi(s_{t-1},y_t)-q_\pi(s_{t-1})
=q_\pi(s_t)-q_\pi(s_{t-1}),
\label{eq:token-advantage}
\end{equation}
which measures a token's contribution to eventual success relative to the expected continuation.

The final binary outcome $Y$ is a Monte Carlo estimator for $q_{\pi}(s_t)$. Also, we can introduce an extra critic to estimate the value function, denoted as $\hat{q}_\phi$. This gives the Monte Carlo advantage estimator:
\begin{equation}
\widehat A_t^{\mathrm{MC}}
=Y-\hat{q}_\phi(s_{t-1}),
\label{eq:mc-advantage}
\end{equation}
where the estimated value function serves only as a pre-token baseline.

With a learned critic, an alternative is to construct the advantage directly from temporal differences:
\begin{equation}
\widehat A_t^{\mathrm{TD}}
=\hat{q}_\phi(s_{t})-\hat{q}_\phi(s_{t-1}).
\label{eq:td-advantage}
\end{equation}
However, this estimator can inherit substantial bias from value-function approximation errors, while reducing variances from the randomness of the outcome reward.

PPO utilizes generalized advantage estimation \citep{schulman2015gae}, which interpolates between temporal-difference estimates and Monte Carlo returns through the parameter $\lambda$. In the undiscounted setting, setting $\lambda=1$ recovers the Monte Carlo estimator above. Thus, the Monte-Carlo advantage estimator can also be denoted as $\widehat A_t^{\mathrm{GAE}(1)}$, which is the default PPO implementation.

\paragraph{Group Relative Policy Optimization.}
GRPO replaces the critic with reward normalization across $G$ responses to the same prompt \citep{shao2024deepseekmath}. Let $\overline Y$ and $\sigma_Y$ denote the mean and standard deviation of their outcomes. Every token in response $i$ receives the same advantage:
\begin{equation}
\widehat A_i^{\mathrm{GRPO}}
=\frac{Y_i-\overline Y}{\sigma_Y+\varepsilon_{\mathrm{std}}}, \quad 1\le i \le G,
\label{eq:grpo-advantage}
\end{equation}
where $\varepsilon_{\mathrm{std}}>0$ ensures numerical stability. For $G=1$, the advantage is zero, requiring a different estimator for single-rollout learning.

Using trajectories sampled from the behavior policy
$\pi_{\mathrm{old}}$, we maximize a clipped surrogate that
limits the incentive for large changes in token probabilities
\citep{schulman2017ppo,shao2024deepseekmath}.
Omitting the KL regularization to the base model, the GRPO objective \citep{liu2025understanding} is
\begin{equation}
\mathcal{J}(\theta)
=
\mathbb{E}_{\substack{
s_0\sim d,\\
y^{(i)}\sim\pi_{\mathrm{old}}(\cdot\mid s_0)
}}
\left[
\frac{1}{G}\sum_{i=1}^{G}
\sum_{t=1}^{T_i}
\min\!\left\{
\rho_{i,t}(\theta)A_i^{\mathrm{GRPO}},
\operatorname{clip}\!\left(
\rho_{i,t}(\theta),1-\epsilon,1+\epsilon
\right)A_i^{\mathrm{GRPO}}
\right\}
\right],
\end{equation}
where $G$ is the group size, $T_i$ is the length of response $i$,
and the importance sampling ratio to correct the distribution shift is defined as $\rho_{i,t}(\theta)
=
\cfrac{\pi_\theta(y_t^{(i)}\mid s_{t-1}^{(i)})}
{\pi_{\mathrm{old}}(y_t^{(i)}\mid s_{t-1}^{(i)})}.$

\paragraph{Pass@$k$ Policy Optimization.}
With the defined value function, standard expected-return policy optimization maximizes Pass@1, with the objective written as $J_\pi(\theta)=\sum_{s_0}d(s_0)q_\pi(s_0)$.

Pass@$k$ optimization instead maximizes the probability that at least one of $k$ independently sampled solutions succeeds, with the following objective:
\begin{equation}
J_{\pi,k}(\theta)
:=\sum_{s_0}d(s_0)q_{\pi,k}(s_0)
=\sum_{s_0}d(s_0)
\left[1-\bigl(1-q_\pi(s_0)\bigr)^k\right].
\label{eq:passk-objective}
\end{equation}
Also, \citeauthor{dekoninck2026beyond} (\citeyear{dekoninck2026beyond}) establish the transformation between the Pass@\(1\) and Pass@\(k\) value functions, which yields the second equality in our objective.
We state the corresponding token-level policy gradient theorem in Section \ref{sec:method} and use it to construct a PPO-style surrogate objective.

\section{Related Work}

\paragraph{Single-Rollout RL}
A recent paper develops Single-Rollout Asynchronous Optimization (SAO) which studies the limitations of group-wise sampling in asynchronous agentic RL \citep{hou2026single} and explores the single-rollout setup. SAO includes more frequent critic updates, frozen attention modules, observation-skipping GAE, and scaled value pretraining to stabilize training. SAO shares our motivation for avoiding repeated sampling. However, our emphasis is on learning with an online critic without large-scale critic pretraining so that it can adapt to frontier reasoning tasks and study the effect of Pass@$k$ gradient weighting in stablizing training.

\paragraph{Dense Credit Assignment.}
Several methods exploit structure across sampled trajectories to obtain fine-grained credit. TEMPO converts groups of completions into prefix trees, estimates nonparametric values from descendant outcomes, and augments GRPO with branch-gated temporal-difference corrections \citep{tran2025treecredit}. TreePO combines dynamic tree sampling, fixed-length segment decoding, and tree-based segment-level advantages to share computation across prefixes while preserving exploration \citep{li2025treepo}. For long-horizon agents, GiGPO combines trajectory-level advantages with action-level comparisons at repeated anchor states \citep{feng2025gigpo}, while HGPO aggregates advantages across hierarchical groups to address historical-context inconsistency \citep{he2026hierarchyofgroups}. VIMPO instead offers a critic-free but value-based approach by deriving policy-implied values from the optimality conditions of KL-regularized RL \citep{kang2026vimpo}. These methods improve credit assignment within a multi-rollout setting, whereas our target setting does not provide the branching or grouping structure they use.

Another growing line of work studies process-level reward modeling for reasoning \citep{cobbe2021training,lightman2023letsverify,wang2023mathshepherd,cui2025prime}. AgentPRM extends progress-based evaluation to agentic settings, where intermediate actions are assessed by their contribution toward the final goal rather than by isolated correctness \citep{xi2025agentprm}. Likelihood-based rewards provide another form of dense supervision by rewarding the probability assigned to a reference answer or continuation \citep{kwiatkowski2026likelihood}. These approaches primarily focus on learning or constructing informative intermediate reward signals. In contrast, our focus is on estimating token-level policy advantages, which determines the direction and magnitude of policy updates. Rather than improving the token-level credit assignment itself, we further study how to stabilize policy optimization when credit estimates are imperfect, particularly in the single-rollout regime.

\paragraph{Pass@$k$ Policy Optimization.}
Several studies directly optimize multi-sample success objectives such as Pass@$k$ sucess rate to encourage exploration \citep{chen2025pass,takashiro2026advantage}.
Pass-at-$k$ Policy Optimization (PKPO) derives unbiased estimators for Pass@$k$ policy gradients using more than $k$ samples per prompt \citep{walder2026pass}. Beyond Accuracy derives a direct policy-gradient reweighting that emphasizes low-success prompts but the estimator still requires multiple rollouts \citep{mahdavi2025beyond}. 
Advantage Shaping as Surrogate Reward Maximization connects direct Pass@$k$ gradients with GRPO-style advantage shaping through their underlying population objectives without introducing any new estimators, interpreting hard-example reweighting as reward-level regularization \citep{thrampoulidis2026advantage}. Our contribution is a new Pass@$k$ gradient estimator formulation in single-rollout learning with an online prefix critic. We further study how Pass@$k$ reweighting interacts with token-level credit assignment to improve robustness to critic error, without requiring additional per-prompt rollouts or large-scale critic pretraining.



\section{Method}
\label{sec:method}
This section presents our single-rollout training framework, which combines Pass@$k$ policy optimization with Monte Carlo advantage estimation. For completeness, we restate the token-level policy gradient theorem with its Pass@$k$ extension, specializing established derivations to undiscounted setting with binary terminal rewards. We then introduce a critic-based estimator for the Pass@$k$ policy gradient, formulate the corresponding PPO-style clipped surrogate \citep{schulman2017ppo}, and describe the critic design.

\subsection{Token-Level Pass@k Policy Gradient}
We follow the trajectory-level likelihood-ratio derivation of REINFORCE \citep{willams1992simple}. 
This formulation operates directly on complete rollouts and their final outcome rewards, 
with prefix-visitation probabilities implicit in the trajectory expectation. 
Subtracting a prefix-dependent baseline from the outcome reward yields an unbiased gradient estimator with Monte-Carlo advantage introudced in Equation \ref{eq:mc-advantage}.
The Pass@$k$ policy gradient is obtained based on the established gradient reweighting \citep{takashiro2026advantage},  
by differentiating the probability transformation in Equation~\ref{eq:passk-value}. 
The following lemma states a token-level Pass@$k$ policy gradient and provides the foundation for single-rollout policy gradient estimator.

\begin{lemma}[Token-Level Pass@$k$ Policy Gradient]
\label{lem:passk-policy-gradient}
Let $\pi=\pi_\theta$ be a differentiable autoregressive policy
with bounded trajectory lengths $T$. Given the prompt distribution $d$ and the
binary reward $Y(s_0,y_{1:T})\in\{0,1\}$, for an integer $k\geq1$ and any bounded suffix-independent baseline $b$,
the Pass@$k$ policy gradient satisfies
\begin{equation}
\nabla_\theta J_{\pi,k}(\theta)
=
\mathbb{E}\!\left[
w_{\pi,k}(s_0)
\sum_{t=1}^{T}
\nabla_\theta\log\pi_\theta(y_t\mid s_{t-1})
\bigl(Y-b(s_{t-1})\bigr)
\right],
\label{eq:token-passk-gradient}
\end{equation}
where the expectation is over $s_0\sim d$ and
$y_{1:T}\sim\pi_\theta(\cdot\mid s_0)$, and
\[
w_{\pi,k}(s_0)
:=
k\bigl(1-q_\pi(s_0)\bigr)^{k-1}.
\]
\end{lemma}

For $k>1$, the weight $w_{\pi,k}(s_0)$ in the policy gradient decreases with prompt success probability, assigning greater relative weight to harder prompts. The proof is written in Appendix \ref{sec:proof_of_policy_gradient_thm}.

\subsection{Pass@k Gradient Estimator with Monte-Carlo Advantage}
Our algorithm relies on an extra trained critic, denoted as $\hat{q}_\phi$.
With the critic, the Pass@$k$ weights can be directly estimated as $\hat{w}_{\pi,k}(s_0)=k\bigl(1-\hat{q}_\phi(s_0)\bigr)^{k-1},$ without needing multiple rollouts to approximate the Pass@k success rate.
Also, we instantiate the baseline with the learned prefix critic,
$b(s_{t-1})=\hat{q}_\phi(s_{t-1})$, obtaining the Monte Carlo advantage estimator shown in Equation \ref{eq:mc-advantage}.
Here, the value function serves only as a baseline and the return is supplied by the observed outcome.

Given a batch $\mathcal{B}$ of prompts, SR-PPO first computes the Monte Carlo advantages and
plug-in Pass@$k$ weights for each prompt $i$, denoted as $\widehat w_{i,k}=k\bigl(1-\hat{q}_{\phi}(s_0^{(i)})\bigr)^{k-1}$ and $\widehat A_{i,t}^{\mathrm{MC}}
=Y^{(i)}-\hat{q}_{\phi}(s_{t-1}^{(i)})$.
Then, the gradient estimator for Pass@$k$ equals
\begin{equation}
    \frac{1}{|\mathcal{B}|} \sum_{i=1}^{|\mathcal{B}|} \widehat w_{i,k} \sum_{t=1}^{T^{(i)}} \nabla \log \pi_\theta(y^{(i)}_t|s_{t-1}^{(i)}) \bigr ( Y^{(i)} - \hat{q}_\phi(s_{t-1}^{(i)})\bigl).
\end{equation} 

Our method adopts PPO-style gradient clipping \citep{schulman2017ppo,shao2024deepseekmath} when collecting
rollouts under $\pi_{\mathrm{old}}$ in the off-policy setting.
The surrogate policy optimization objective for gradient clipping equals
\begin{equation}
\mathcal{L}_{\mathrm{actor}}(\theta)
=
-\frac{1}{|\mathcal{B}|}
\sum_{i\in\mathcal{B}}\widehat w_{i,k}
\sum_{t=1}^{T_i}
\min\Bigl\{
\rho_{i,t}(\theta)\widehat A_{i,t}^{\mathrm{MC}},
\operatorname{clip}\!\left(
\rho_{i,t}(\theta),1-\epsilon,1+\epsilon
\right)\widehat A_{i,t}^{\mathrm{MC}}
\Bigr\},
\end{equation}
with the importance sampling ratio $\rho_{i,t}(\theta) = \cfrac{\pi_\theta(y^{(i)}_t|s_{t-1}^{(i)})}{\pi_\mathrm{old}(y^{(i)}_t|s_{t-1}^{(i)})}$ and the clipping parameter $\epsilon>0$.

\subsection{Critic Design}
When training our separate critic network, denoted as $\hat{q}_\phi$, 
each nonterminal prefix is supervised by the observed binary outcome of its completed trajectory. 
Given a rollout batch $\mathcal{B}$, the critic is trained with binary cross-entropy augmented by an
auxiliary mean-squared-error term:
\begin{equation}
\begin{aligned}
\mathcal{L}_{\mathrm{critic}}(\phi;\mathcal{B})
=
\frac{1}{\sum_{i\in\mathcal{B}}T^{(i)}}
\sum_{i\in\mathcal{B}}\sum_{t=1}^{T^{(i)}}
\Bigl[
&-Y^{(i)}\log\hat q_\phi(s_{t-1}^{(i)})-(1-Y^{(i)})\log\bigl(1-\hat q_\phi(s_{t-1}^{(i)})\bigr)\\
&+\lambda_{\mathrm{cal}}
\bigl(\hat q_\phi(s_{t-1}^{(i)})-Y^{(i)}\bigr)^2
\Bigr],
\end{aligned}
\label{eq:critic-loss}
\end{equation}
where $\lambda_{\mathrm{cal}}\geq0$ controls the calibration level so that the logit output represents the true probability value. The critic is trained online using the same rollout
data collected for policy optimization but after the actor update to get an effective baseline.

SR-PPO differs from our PPO baseline in both Pass@$k$ weighting and critic design. Whereas the PPO baseline trains an unconstrained value head using mean-squared error alone, SR-PPO uses a sigmoid output and the combined binary cross-entropy and squared-error loss. The sigmoid parameterization ensures that $\hat q_\phi(s)\in(0,1)$, allowing critic predictions to be used directly as success-probability estimates in Pass@$k$ reweighting. 

\section{Experiment Results}
We evaluate SR-PPO against PPO and GRPO, with particular emphasis on whether Pass@$k$ reweighting stabilizes learning under misleading credit assignment induced by an imperfect critic. All evaluation metrics in this section are introduced in Appendix \ref{appendix:evaluation-metrics}.

\subsection{Experimental Setup}

We conduct mathematical reasoning experiments using Qwen3-1.7B \citep{yang2025qwen3} in thinking mode as the base model, with a maximum response length of \(9{,}182\) tokens. Training is performed on the DeepScaleR mathematical reasoning dataset \citep{deepscaler2025}. Our goal is to identify a stable single-rollout credit-assignment method that improves the reasoning capabilities of a small open-source base model, rather than to establish the advantages of single-rollout training in agentic tasks, as investigated by \citet{hou2026single}.

To approximately match per-update computation, GRPO uses a mini-batch size of \(64\) with \(4\) rollouts per prompt, whereas SR-PPO and PPO use a mini-batch size of \(128\) with a single rollout per prompt. For SR-PPO, we halve the effective actor mini-batch size when accounting for the additional critic update cost. Other training hyperparameters are provided in the Appendix \ref{appendix:hyper}.


\subsection{Math Reasoning Perforamnce}

\begin{figure}
\centering
\includegraphics[width=\linewidth]{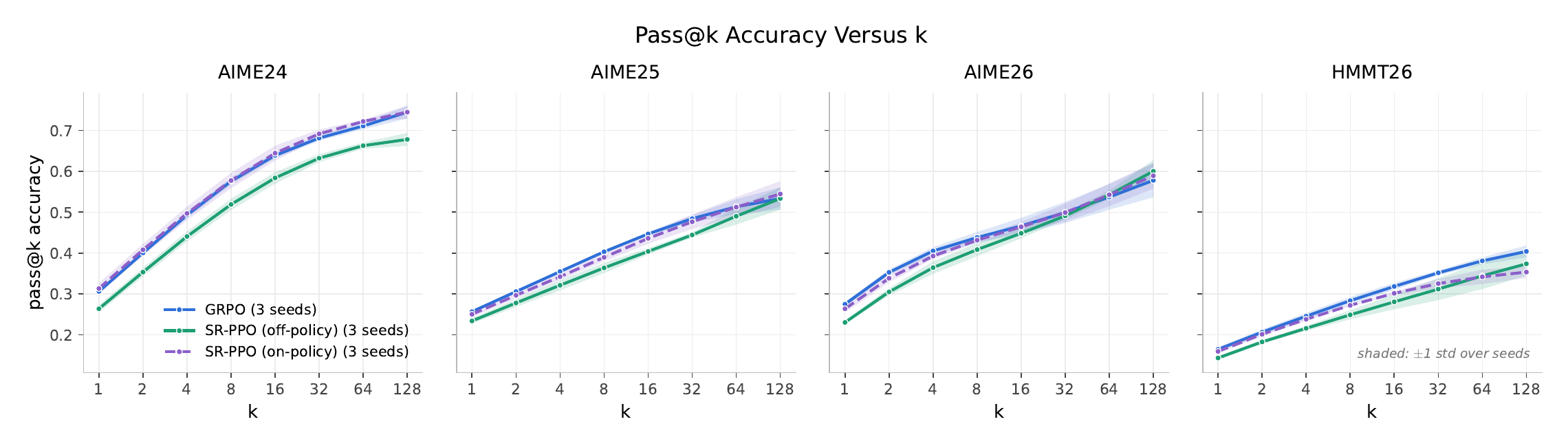}
\caption{Pass@$k$ accuracy on four competition-math benchmarks. 
GRPO and SR-PPO (off-policy) are evaluated at step~150, SR-PPO (on-policy) at step~600. 
Each benchmark has 30 prompts (AIME) or 33 (HMMT26), so a single problem is worth roughly 3 accuracy points.}
\label{fig:math-test}
\vspace{-1.5em}
\end{figure}

Our goal is to identify a stable single-rollout credit-assignment method that improves the reasoning capabilities of a small open-source base model, rather than to establish the advantages of single-rollout training in agentic tasks, as investigated by \citet{hou2026single}.

In the on-policy setting, the train data batch size is used to perform one gradient update per step and thus all importance sampling ratio remains one, without triggering the clipping, while the off-policy algorithm performs four gradient updates per step. Thus, the on-policy SR-PPO is evaluated at step $600$, which gives the same number of gradient updates.
 
All three methods show the expected monotone increase in Pass@$k$ with k, rising steeply from k = 1 to k = 16 and flattening thereafter as the estimator exhausts the 128 available rollouts. On-policy SR-PPO and GRPO are close to indistinguishable on four benchmarks, while off-policy SR-PPO is a weaker configuration. The largest gap between off-policy SR-PPO and GRPO is on AIME24, where it trails GRPO by about $4\%$ at pass@1 and $7\%$ at pass@128. Its one favourable result is on AIME26 at large k, where it overtakes other two methods.
In conclusion, no method is separated from the others by a margin that clearly exceeds seed and sampling noise.

\subsection{Held-Out Reasoning and Logic Test}

\begin{table}[h]
\centering
\begin{tabular}{lrrr}
\toprule
\textbf{Metric} & \textbf{Knights \& Knaves} & \textbf{GPQA-D } & \textbf{Countdown4}\\
\midrule
\textsc{Base Model}  & $100.00$ & $46.46$ & $48.44$ \\
\textsc{GRPO} & $100.00 \pm 0.00 $ & $65.49 \pm 3.55$ & $\mathbf{56.90 \pm 2.35}$   \\
\textsc{On-Policy SR-PPO}  & $100.00 \pm 0.00$ & $\mathbf{67.68 \pm 2.67}$ & $55.34 \pm 2.60 $ \\
\textsc{Off-Policy SR-PPO} & $100.00 \pm 0.00$ & $61.28 \pm 4.85$ & $55.08 \pm 2.71$   \\
\bottomrule
\end{tabular}
\caption{Best@$16$ success rate on three held-out reasoning benchmarks, reported as
mean $\pm$ standard deviation over three seeds. The base model is evaluated
once and has no seed variation. Knights~\&~Knaves is saturated at $100\%$ for
all models.}
\label{tab:logic-perf}
\vspace{-1.0em}
\end{table}

To assess whether improvements from mathematical training
transfer beyond the training distribution, we evaluate all methods
on three held-out benchmarks (Table~\ref{tab:logic-perf}), introduced in Appendix \ref{appendix:test}.
Knights~\&~Knaves is saturated, where the base model
already solves every instance. On the two discriminative
benchmarks, every trained method improves substantially over the base model, confirming
that the reasoning ability acquired during math training is not confined to
the training domain. Among the trained methods the differences are small
relative to seed variation. On-policy SR-PPO attains the highest GPQA-D
accuracy, while GRPO is nominally best on Countdown4, but the
standard deviations overlap substantially. We therefore read these results as showing that
on-policy SR-PPO transfers on par with GRPO.


\subsection{On-Policy SR-PPO Variant}
Figure~\ref{fig:on-policy-variant}(a) shows that both Monte Carlo advantage estimation and Pass@$k$ reweighting are important for stable single-rollout learning, even in the on-policy setting. Replacing it with one-step temporal-difference advantages causes training to collapse, suggesting that critic prediction errors can introduce substantial bias into the policy gradient. The noise-to-signal reduction effect from Pass@$k$ reweighting is thoroughly studied in the next subsection. Figure~\ref{fig:intro}(b) shows that SR-PPO learns steadily in a prolonged training to $600$ steps. 

\begin{figure}
\centering
\includegraphics[width=\linewidth]{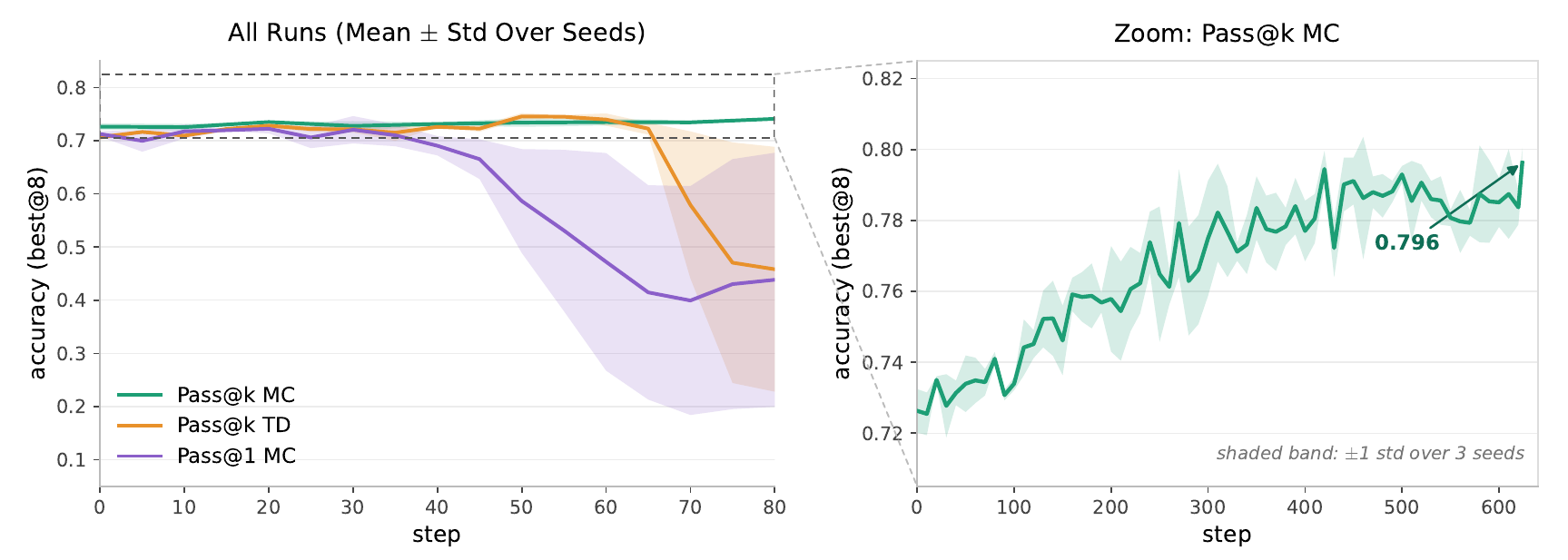}
\caption{\textbf{Ablation study of on-policy SR-PPO.}
\textbf{(a)} Both Monte Carlo advantage estimation and Pass@$k$ optimization are important for stable single-rollout training. Using instead a temporal-difference estimator or Pass@1 optimization causes training to collapse around step 80.
\textbf{(b)} SR-PPO remains stable over two epochs of prolonged training. }
\label{fig:on-policy-variant}
\vspace{-1.5em}
\end{figure}

\subsection{Gradient Noise Analysis}
\label{sec:noise_analysis}
All the critic quality and noise analysis are performed using a checkpoint from on-policy SR-PPO at step $180$, while the rollouts are generated from a GRPO policy.

To quantify the stability of stochastic policy-gradient estimates, we measure the
noise-to-signal ratio (NSR), defined as the gradient variance relative to the
squared norm of the mean gradient. Given $R$ independent rollout replicas with
gradients
\[
g^{(r)}=\mu+\varepsilon^{(r)}, \, 1\le r\le R,
\qquad
\mu=\mathbb{E}[g^{(r)}],
\]
the population NSR is
\[
\operatorname{NSR}
=
\frac{\mathbb{E}\!\left[\|\varepsilon\|_2^2\right]}
{\|\mu\|_2^2}
=
\frac{\operatorname{tr}\operatorname{Cov}(g)}
{\|\mathbb{E}[g]\|_2^2}.
\]
A lower NSR indicates that stochastic gradients are more concentrated around
their expected update direction, while a higher NSR indicates stronger
rollout-induced noise. The finite-sample estimator for NSR is detailed in
Appendix~\ref{appendix:evaluation-metrics}.

Figure~\ref{fig:intro}(c) reports NSR for the gradient aggregated over the full
batch, while Figure~\ref{fig:intro}(a) and
Table~\ref{tab:noisy-level} analyze gradients at the prompt level.
Pass@$k$ reweighting reduces the batch-level NSR from $28.57$ to $10.36$, a
$63.7\%$ reduction, providing a direct explanation for its improved training
stability.

\begin{table}[ht]
\centering
\begin{tabular}{lrrr}
\toprule
\textbf{Metric} & \textbf{Easy} & \textbf{Medium} & \textbf{Hard} \\
\midrule
\textsc{Number of Prompts}
    & 73 & 16 & 39 \\
\textsc{SR-PPO per-prompt NSR mean}
    & 177.915 & 78.188  & 27.496 \\
\textsc{SR-PPO per-prompt NSR median}
    & 55.16 & 49.02 & 15.53 \\
\textsc{SR-PPO Safe BF16 Changes}
    & 99.920\% & 99.894\% & 99.893\% \\
\bottomrule
\end{tabular}
\caption{
\textbf{Per-prompt gradient noise-to-signal ratio and parameter-update analysis by problem
difficulty.}
Prompts are grouped according to the number of successful responses among
eight rollouts: easy ($\geq6$), medium ($3$--$5$), and hard ($\leq2$).
SR-PPO exhibits its lowest NSR on hard prompts. The last row report the fraction of
BF16-realized parameter changes that fall within the safe region.
}
\label{tab:noisy-level}
\vspace{-1.5em}
\end{table}

To understand where this reduction comes from, we group prompts according to
their empirical difficulty. 
We classify a prompt as \emph{hard} if at most two rollouts are correct, \emph{medium} if three to five are correct, and \emph{easy} if at least six are correct. 
For SR-PPO the hard prompts provide the least noisy updates, as shown in Table~\ref{tab:noisy-level}. Consistent with this observation, Pass@$k$ reweighting increases the fraction of total update weight
assigned to hard prompts from $30.5\%$ to $65.8\%$, while reducing the contribution of easy prompts from $57.0\%$ to $24.3\%$, shown in Figure~\ref{fig:intro}(c). Thus, Pass@$k$ does not merely rescale the
gradient; it changes the effective training distribution toward prompts whose
SR-PPO gradients have substantially lower NSR ratio.

We further investigate whether noisy gradients destabilize training by modifying parameters associated with the principal structure of the pretrained model. 
We follow the definition of a \emph{safe region} in Zhu et al. (\citeyear{zhu2025path}) that excludes parameters with large contributions to the principal low-rank component while retaining low-magnitude parameters. 
For an parameter update, we measure the fraction of BF16-realized parameter changes that fall within this region and the details of the evaluation metric is introduced in Appendix \ref{appendix:evaluation-metrics}.

We find that all difficulty groups place near $100\%$ fraction  of their change in the safe region, whose parameter density is $77.06\%$. 
These results do not support the hypothesis that noisy updates cause collapse by moving into the principal parameter subspace. 
Instead, the instability appears to arise mainly from the direction and consistency of
the stochastic gradients, rather than from a qualitatively different region of parameter space being updated. 
Additional cosine-similarity and policy-update analyses in Appendix \ref{appendix:gradient_cosine} show that easy-prompt gradients can also partially cancel the effective update directions contributed by hard prompts, though they give limited gain to the model.

We investigate critic-related sources of gradient noise in Appendix~\ref{appendix:critic_quality}. 
The results suggest that the critic retains the ability to distinguish prompt difficulty,
despite its near-chance ability to judge correct or incorrect sentences within individual trajectories. This may explain why an imperfect critic can still support learning: coarse estimates
of prompt difficulty provide useful baseline information, analogous to the group-mean baseline, even when fine-grained credit estimates are unreliable.

We hypothesize that critic errors are consequential for failed rollouts on easy prompts. These trajectories often contain mostly correct reasoning shown in Table \ref{tab:critic-difficulty-bin}, interrupted by a few decisive mistakes, making reliable credit assignment dependent on precise error localization. An inaccurate critic may continue to assign high values to unrecoverable errors. Successful rollouts, by contrast, exhibit closer agreement between the terminal outcome and the quality of intermediate reasoning with recoverable mistakes. This asymmetry may make failed easy-prompt rollouts especially sensitive to critic imprecision, contributing to the higher gradient noise observed for easy prompts.

\section{Conclusion}

We presented SR-PPO, a single-rollout training framework that combines Monte Carlo token-level advantage estimation with Pass@$k$ gradient reweighting. 
A prefix critic is trained online from terminal outcomes and used both as a baseline and to estimate prompt-level reweighting coefficients. 
Our mathematical reasoning experiments suggest that this combination supports stable and prolonged policy optimization despite imperfect critic estimates.

Beyond evaluating training performance, we analyzed critic quality, estimation errors, gradient noise, and update behavior.
The critic captures coarse differences in prompt difficulty despite providing unreliable within-rollout mistake localization.
These findings help explain why critic-based reweighting remains useful: it reduces the contribution of noisy easy-prompt updates and their interference with gradients from harder prompts.
Our analysis thus provides an empirical account of both the stabilization mechanism and the residual credit-estimation errors that may limit performance.

Our conclusions are restricted to the DeepScaleR mathematical reasoning dataset, a Qwen3-1.7B base model, and a limited hyperparameter search. Generalization to other models, domains, and long-horizon agentic environments, as well as further investigation on critic training objectives and advantage normalization techniques, remains unexplored.
Given computational constraints, we focus on understanding and stabilizing single-rollout learning rather than comprehensively establishing its advantages over multi-rollout training, as investigated by \citet{hou2026single}. 

\section*{Reproducibility Statement}
All the code and checkpoints will be released after the double blind period. All the hyperparameters and design choices are already covered in the paper to support reproducibility and transparency.

\section*{AI Use Statement}
Generative AI tools assist manuscript writing, coding and figure generation. The authors are responsible for the proof, citations, algorithm design, experiment results, interpretation, and final content.

\bibliography{iclr2027_conference}
\bibliographystyle{iclr2027_conference}

\clearpage
\appendix
\section{Pass@k Policy Gradient Theorem Proof}
\label{sec:proof_of_policy_gradient_thm}

\begin{proof}
Using $q_{\pi,k}(s_0)=1-(1-q_\pi(s_0))^k$,
the chain rule gives
\[
\nabla_\theta J_{\pi,k}(\theta)
=
\sum_{s_0}d(s_0)w_{\pi,k}(s_0)
\nabla_\theta q_\pi(s_0),
\]
establishing the first equality.

For each prompt, the likelihood-ratio identity and
autoregressive factorization yield
\begin{align*}
\nabla_\theta q_\pi(s_0)
&=
\mathbb{E}_{y_{1:T}\sim\pi_\theta(\cdot\mid s_0)}
\left[
Y\nabla_\theta\log\pi_\theta(y_{1:T}\mid s_0)
\right]\\
&=
\mathbb{E}_{y_{1:T}\sim\pi_\theta(\cdot\mid s_0)}
\left[
Y\sum_{t=1}^{T}
\nabla_\theta\log\pi_\theta(y_t\mid s_{t-1})
\right].
\end{align*}
A fixed pre-action baseline has zero expected contribution
conditional on the prefix:
\begin{align*}
&\mathbb{E}_{y_t\sim\pi_\theta(\cdot\mid s_{t-1})}
\left[
b(s_{t-1})
\nabla_\theta\log\pi_\theta(y_t\mid s_{t-1})
\right]\\
&\qquad=
b(s_{t-1})
\nabla_\theta\sum_{y_t}\pi_\theta(y_t\mid s_{t-1})
=0.
\end{align*}
For a fixed policy, $w_{\pi,k}(s_0)$ is determined by the
prompt and therefore preserves this cancellation.
Subtracting the baseline, multiplying by the prompt-level
weight, and averaging over prompts gives the second equality.
\end{proof}

\section{Hyperparamters}
\label{appendix:hyper}


Due to the limited computational resources, most hyperparamters for math reasoning use the default or common setup for LLM post-training as shown in Table \ref{tab:math-hyper}. But the batch size is differently set up, especially for off-policy and on-policy variants, shown in Table \ref{tab:math-hyper-bs}.

\begin{table}[ht]
\centering
\begin{tabular}{c|c|c|c|c|c}
\toprule
\textbf{Actor LR} & \textbf{Critic LR} & \textbf{Clip Ratio High} & \textbf{Entropy Coef.} & \textbf{KL Coef.} & \textbf{Temperature} \\
\midrule
$1e-6$ & $1e-5$ & $0.25$ & $0.0$ & $0.001$ & $1.2$\\
\bottomrule
\end{tabular}
\caption{
Shared and untuned hyperparamters between all algorithms for math reasoning.
}
\label{tab:math-hyper}
\end{table}

\begin{table}[ht]
\centering
\begin{tabular}{c|c|c|c}
\toprule
& \textbf{Train Data BS} & \textbf{Per-Gradient-Update BS} & \textbf{Rollout Number} \\
\midrule
\textbf{On-Policy SR-PPO/PPO} & $128$ & $128$ & $1$\\
\textbf{Off-Policy GRPO} & $256$ & $64$ & $4$\\
\textbf{Off-Policy SR-PPO/PPO} & $512$ & $128$ & $1$\\
\bottomrule
\end{tabular}
\caption{
Different batch size and rollout number setting for a fair comparison.
}
\label{tab:math-hyper-bs}
\end{table}

The SR-PPO specific hyperparamter is set to be $k=4$ and $\lambda_\mathrm{cal} = 1.0$.

\section{Evaluation Metrics}
\label{appendix:evaluation-metrics}
\paragraph{Pass@k Accuracy}
For \(k<128\), Pass@\(k\) is estimated from these samples using the standard unbiased estimator
$$
\widehat{\mathrm{Pass@}k}
=
1-
\frac{\binom{n-c}{k}}{\binom{n}{k}},
\qquad n=128,
$$

where \(c\) denotes the number of correct completions among the \(n\) sampled responses.

\paragraph{Noise-to-Signal Ratio}
To quantify the stability of stochastic policy-gradient estimates, we measure the noise-to-signal ratio (NSR), defined as the gradient variance relative to the squared norm of the mean gradient. For a stochastic gradient estimator \(\hat g\) with mean \(\mu=\mathbb{E}[\hat g]\),

$$
\operatorname{NSR}
=
\frac{\mathbb{E}\!\left[\|\hat g-\mu\|_2^2\right]}
{\|\mu\|_2^2}
=
\frac{\operatorname{tr}\operatorname{Cov}(\hat g)}
{\|\mathbb{E}[\hat g]\|_2^2}.
$$

A lower NSR indicates that stochastic updates are more consistently aligned with the expected optimization direction, whereas a higher NSR reflects stronger rollout-induced noise.

Given \(R\) independent gradient replicas \(g^{(1)},\ldots,g^{(R)}\), we estimate the signal magnitude by

$$
\widehat S
=
\frac{1}{R(R-1)}
\sum_{r\neq s}
\left\langle g^{(r)},g^{(s)}\right\rangle,
$$

and the noise magnitude by

$$
\widehat V
=
\frac1R\sum_{r=1}^R
\|g^{(r)}\|_2^2
-
\widehat S.
$$

The empirical NSR is then

$$
\widehat{\operatorname{NSR}}
=
\frac{\widehat V}{\widehat S},
$$

with \(\widehat{\operatorname{NSR}}=\infty\) when \(\widehat S\le \epsilon\) due to an unreliable finite-sample signal estimate.

\paragraph{Off-Principal Ratio}

We measure whether gradient updates concentrate on parameters that are less structurally important to the pretrained checkpoint, following the design in \citeauthor{zhu2025path} (\citeyear{zhu2025path}). For each analyzed weight matrix \(W_0\), we first compute a rank-\(128\) approximation
$$
W_0 \approx U_{128}\Sigma_{128}V_{128}^{\top},
$$
which keeps weights associatd with top $128$ singular values.

We define the principal-region mask as the \(30\%\) of entries with the largest absolute magnitude in this low-rank reconstruction:

$$
M_{\mathrm{princ}}
=
\operatorname{Top}_{30\%}
\left(
\left|U_{128}\Sigma_{128}V_{128}^{\top}\right|
\right).
$$

We separately define the low-magnitude region as the \(30\%\) of entries with the smallest absolute pretrained weights:

$$
M_{\mathrm{low}}
=
\operatorname{Bottom}_{30\%}(|W_0|).
$$

The safe region is then
$M_{\mathrm{safe}}
=
M_{\mathrm{princ}}^{c}\cup M_{\mathrm{low}}
=
\left(
M_{\mathrm{princ}}\cap M_{\mathrm{low}}^{c}
\right)^c.
$

Intuitively, this mask excludes parameters that simultaneously have large contribution to the principal low-rank structure and are not small in magnitude. In the analyzed checkpoint, \(M_{\mathrm{safe}}\) contains \(77.06\%\) of the considered parameters.

For each difficulty bin
\(B\in\{\mathrm{easy},\mathrm{medium},\mathrm{hard}\}\),
we construct a gradient update

$$
\Delta W_B=\eta g_B,
\qquad
W_B^{+}=W_0+\Delta W_B,
\qquad
\eta=0.01.
$$

Because training is performed in finite precision, we distinguish the continuous update from the subset of entries whose stored BF16 value actually changes. We define

$$
M_{\mathrm{update},B}
=
\left\{
(a,b):
\operatorname{bf16}(W_{B,ab}^{+})
\neq
\operatorname{bf16}(W_{0,ab})
\right\}.
$$

We report the \emph{changed-weight safe-region rate} measures the fraction of BF16-realized parameter changes that occur inside the safe region:

$$
R_{\mathrm{safe},B}^{\mathrm{change}}
=
\frac{
|M_{\mathrm{update},B}\cap M_{\mathrm{safe}}|
}{
|M_{\mathrm{update},B}|
}.
$$

This metric captures where the update produces an actual change in the stored parameters.

\subsection{Critic Quality AUC}
AUC measures how well the critic ranks correct sentences above incorrect sentences, without requiring its scores to be calibrated probabilities. This is appropriate because the ground-truth judgments are binary, while the critic outputs arbitrary continuous values.

AUC measures how well the critic ranks correct sentences above incorrect sentences, without requiring its scores to be calibrated probabilities. An AUC of $1.0$ indicates perfect ranking, $0.5 $chance-level ranking, and values below $0.5$ indicate that incorrect examples tend to receive higher scores. Ties receive half credit. 

For each sentence \(j\) in prompt \(i\), we average the critic values over all tokens in that sentence,

$$
s_{ij}=\frac{1}{|T_{ij}|}\sum_{t\in T_{ij}}\hat{q}_\phi(s_t),
$$

and assign a binary correctness label \(y_{ij}\in\{0,1\}\) from the Codex judgment. We use the critic's raw score, with no clipping or thresholding.

Within each prompt, AUC can be interpreted as the probability that a randomly selected correct sentence receives a higher critic score than a randomly selected incorrect sentence, with ties receiving half credit:

$$
\mathrm{AUC}_i
=
\frac{
\sum_{j:y_{ij}=1}\sum_{k:y_{ik}=0}
\left[
\mathbf 1(s_{ij}>s_{ik})
+\frac12\mathbf 1(s_{ij}=s_{ik})
\right]
}{
n_i^+n_i^-
},
$$
where $n_i^+$ is the number of positive/correct sentences in prompt \(i\), and \(n_i^-\) is the number of negative/incorrect sentences.
So $ n_i^+ n_i^- $ is simply the total number of correct–incorrect sentence pairs being compared.

Then, we average \(\mathrm{AUC}_i\) equally across prompts. Prompts containing only correct or only incorrect sentences are excluded because within-prompt AUC is undefined for them, and we report the number of prompts retained. This within-prompt construction focuses the metric on whether the critic can distinguish good from bad reasoning within the same problem.

\section{Further Empirical Results}



\subsection{Reasoning and Logic Benchmarks}
\label{appendix:test}
\emph{Knights~\&~Knaves} tests logical deduction: given statements
from characters who either always tell the truth or always lie,
the model must identify each character's type.
\emph{GPQA Diamond (GPQA-D)} evaluates scientific knowledge and
reasoning through graduate-level multiple-choice
questions in biology, chemistry, and physics.
\emph{Countdown4} tests arithmetic planning by requiring the
model to combine four given numbers using basic arithmetic
operations to reach a specified target.

\subsection{Critic Quality}
\label{appendix:critic_quality}
We use Codex 5.6-Sol to judge both whether each sentence is correct and whether the solution state remains correct after that sentence. Notice that if the solution is already wrong, then after a correct sentence, the solution state remains negative. We then compute AUC separately for these two binary labels. We also experiment with directly training the critic to predict the Pass@\(k\) value function, but observed no significant improvement over a standard Pass@\(1\) critic and therefore do not adopt this design. As shown in Table~\ref{tab:critic-quality}, the resulting per-prompt AUCs are only slightly above \(0.5\), indicating that the critic is only marginally better than random chance at ranking correct states or sentences above incorrect ones within each prompt’s solution. 

However, pooling sentences across all \(128\) prompts yields a higher AUC, suggesting that the critic retains the ability to distinguish quality across prompts despite its limited ability to localize mistakes within individual trajectories. This may explain why an imperfect critic can still support learning: even when fine-grained credit estimates are unreliable, coarse distinctions between easy and hard prompts can remain useful, similar to the group-relative baseline. As shown in Table \ref{tab:critic-value}, the critic outputs' magnitudes are indeed proportional to the success rate.

\begin{table}[ht]
\centering
\begin{tabular}{lrrr}
\toprule
\textbf{Model} & \textbf{SR-PPO} & \textbf{SR-PPO Pass@k Variant} & \textbf{PPO}\\
\midrule
\textsc{Per-Prompt Sentence AUC Mean}  & $0.5362$ & $0.5204$ & $0.4958$ \\
\textsc{Per-Prompt State AUC Mean} & $0.5693$ & $0.5704$ & $0.5370$   \\
\textsc{All Sentence AUC}  & $0.6735$ & $0.6809$ & $0.6409$ \\
\textsc{All State AUC} & $0.7385$ & $0.7424$ & $0.7056$   \\
\bottomrule
\end{tabular}
\caption{
Critic Quality Score.
}
\label{tab:critic-quality}
\end{table}

\begin{table}[h]
\centering
\begin{tabular}{lrrr}
\toprule
\textbf{Metric} & \textbf{Easy} & \textbf{Medium} & \textbf{Hard}\\
\midrule
\textsc{SR-PPO}  & $0.6694\pm 0.2627$ & $0.4498 \pm 0.2758$ & $0.2336 \pm 0.2499$ \\
\textsc{PPO} & $0.6699 \pm 0.2003$ & $0.4546 \pm 0.2216$ & $0.3511 \pm 0.2429$   \\
\bottomrule
\end{tabular}
\caption{
Averaged token values from the critic for all difficulty bins.
}
\label{tab:critic-value}
\end{table}

We also observe a substantial drop in the fraction of valid solution states, from \(84.99\%\) on easy problems to \(48.58\%\) on hard problems. We hypothesize that the higher noise on easy prompts arises because failed rollouts are often mostly valid except for a small number of decisive mistakes, which the critic may fail to localize and may even assign higher values to. In contrast, credit-assignment errors on successful rollouts are less harmful because any intermediate mistakes are ultimately recoverable.

\begin{table}[h]
\centering
\begin{tabular}{lrrr}
\toprule
\textbf{Metric} & \textbf{Easy} & \textbf{Medium} & \textbf{Hard}\\
\midrule
\textsc{Sentence Correctness}  & 94.02\% & 90.22\% & 80.58\% \\
\textsc{Prefix State Validity} & 84.99\% & 71.28\% & 48.58\%   \\
\bottomrule
\end{tabular}
\caption{
Fraction of correct sentences among different difficulty.
}
\label{tab:critic-difficulty-bin}
\end{table}

\subsection{Gradient Direction Analysis}

In Figure \ref{fig:heatmap-cosine} Panel (a), each row is the difficulty bin whose data was used for the update; each column is the bin whose correct-answer log-probability is measured. Two structural features stand out.

First, the diagonal is positive everywhere, but its magnitude falls sharply with easiness, meaning that easy prompts have little left to gain. Hard prompts, by contrast, have the most headroom and the largest self-improvement.

Second, and more importantly, the off-diagonal structure is strongly asymmetric. Updating on hard data never meaningfully hurts the other bins ($+0.00198$ on medium, $-0.00005$ on easy, i.e. neutral). But updating on medium or easy data actively degrades hard-prompt correctness, by $-0.01229$ and $-0.00566$ respectively.

Panel (b) shows that the cosine similarities supply the mechanism. Hard-bin gradients are strongly anti-aligned with both medium ($-0.559$) and easy ($-0.567$), while medium and easy are mutually aligned ($+0.521$). Easier prompts therefore form a coherent group pulling parameters in one direction, and that direction is close to the opposite of what hard prompts require. 

Putting the two together. The negative cosines explain why easy and medium updates push hard-prompt likelihood down, though easy answers sit near the ceiling and have little probability to gain.

\label{appendix:gradient_cosine}
\begin{figure}[h]
\centering
\includegraphics[width=\linewidth]{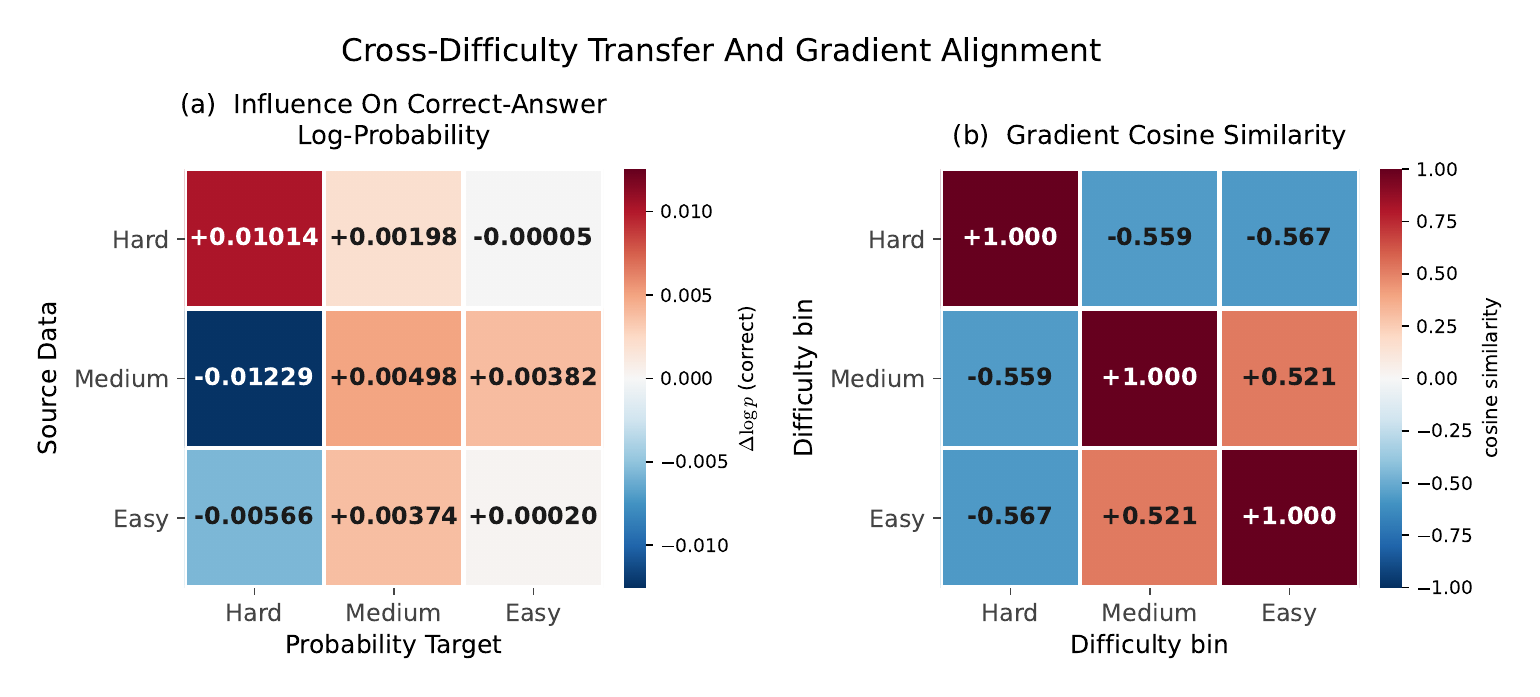}
\caption{\textbf{Gradient direction analysis.}
Panel (a): cross-difficulty influence. Each row is the difficulty bin whose data was used for the update; each column is the bin whose correct-answer log-probability is measured. Panel (b): gradient alignment. The cosine similarities supply the mechanism. Hard-bin gradients are strongly anti-aligned with both medium (-0.559) and easy (-0.567). }
\label{fig:heatmap-cosine}
\vspace{-0.5em}
\end{figure}

\end{document}